\documentclass{article}

\usepackage{arxiv}
\usepackage[utf8]{inputenc}
\usepackage[T1]{fontenc}
\usepackage[hidelinks]{hyperref}
\usepackage{url}
\usepackage{booktabs}
\usepackage{amsfonts}
\usepackage{nicefrac}
\usepackage{microtype}
\usepackage{amsmath,amssymb,amsthm}
\usepackage{graphicx}
\usepackage{subcaption}
\usepackage{xcolor}
\usepackage{natbib}
\newcommand\blfootnote[1]{%
  \begingroup
  \renewcommand\thefootnote{}\footnote{#1}%
  \addtocounter{footnote}{-1}%
  \endgroup
}
\graphicspath{ {./figures/} }

\newtheorem{theorem}{Theorem}
\newtheorem{lemma}{Lemma}
\newtheorem{remark}{Remark}
\newtheorem{corollary}{Corollary}
\newtheorem{proposition}{Proposition}
\theoremstyle{definition}
\newtheorem{definition}{Definition}

\newcommand{\avg}{\text{avg}}
\newcommand{\res}{\text{res}}
\newcommand{\E}{\mathbb{E}\,}
\newcommand{\R}{\mathbb{R}}
\newcommand{\he}{\mathrm{he}}

\title{Emergence of Globally Attracting Fixed Points in\\Deep Neural Networks With Nonlinear Activations}

\author{
  Amir Joudaki \\
  ETH Zürich\\
  \texttt{amir.joudaki@inf.ethz.ch} \\
  \And
  Thomas Hofmann \\
  ETH Zürich\\
  \texttt{thomas.hofmann@inf.ethz.ch }
}

\begin{document}
\maketitle
\blfootnote{Code available at: \url{https://github.com/ajoudaki/kernel-global-dynamics}}

\begin{abstract}
Understanding how neural networks transform input data across layers is fundamental to unraveling their learning and generalization capabilities. Although prior work has used insights from kernel methods to study neural networks, a global analysis of how the similarity between hidden representations evolves across layers remains underexplored. In this paper, we introduce a theoretical framework for the evolution of the kernel sequence, which measures the similarity between the hidden representation for two different inputs. Operating under the mean-field regime, we show that the kernel sequence evolves deterministically via a kernel map, which only depends on the activation function. By expanding activation using Hermite polynomials and using their algebraic properties, we derive an explicit form for kernel map and fully characterize its fixed points. Our analysis reveals that for nonlinear activations, the kernel sequence converges globally to a unique fixed point, which can correspond to orthogonal or similar representations depending on the activation and network architecture. We further extend our results to networks with residual connections and normalization layers, demonstrating similar convergence behaviors. This work provides new insights into the implicit biases of deep neural networks and how architectural choices influence the evolution of representations across layers.
\end{abstract}

\section{Introduction}
Deep neural networks have revolutionized various fields, from computer vision to natural language processing, due to their remarkable ability to learn complex patterns from data. Understanding the internal mechanisms that govern their learning and generalization capabilities remains a fundamental challenge. 

One approach to studying these transformations is through the lens of kernel methods. Kernel methods have a long history in machine learning for analyzing relationships between data points in high-dimensional spaces~\citep{scholkopf2002learning, smola2004tutorial}. They provide a framework for understanding the similarity measures that underpin many learning algorithms. Recent theoretical studies have increasingly focused on analyzing neural networks from the perspective of kernels. The Neural Tangent Kernel (NTK)~\citep{jacot2018neural} is a seminal work that provided a way to analyze the training dynamics of infinitely wide neural networks using kernel methods. This perspective has been further explored in various contexts, leading to significant advances in our understanding of neural networks~\citep{lee2019wide, arora2019exact, yang2019scaling}.

Despite these advances, an important question remains unexplored: \emph{How does the similarity between hidden layer representations evolve across layers, and how is that affected by particular choices of nonlinear functions?} Previous work has mainly focused on local behaviors or specific initialization conditions \citep{saxe2013exact, schoenholz2016deep, pennington2017resurrecting}. A comprehensive global analysis of neural kernel sequence fixed points and convergence properties, particularly in the presence of nonlinear activations, is still incomplete.

This paper addresses this gap by introducing and analyzing the evolution of kernel sequences in deep neural networks. Specifically, we consider the kernel sequence $ k (h_\ell (x), h_\ell (y)) $, where $ k$ denotes a similarity measure, and $ h_\ell (x) $ and $ h_\ell (y) $ are representations of the inputs $ x$ and $ y$ at layer $ ell$. Understanding whether and how this sequence converges to a fixed point as the depth of the network increases is crucial for uncovering the inherent implicit biases of deep networks.

Our analysis builds upon foundational work in neural network theory and leverages mean-field theory to simplify the analysis. By considering the infinite-width limit, stochastic sequences become deterministic, allowing us to focus on the underlying dynamics without the interference of random fluctuations \citep{poole2016exponential, yang2019meanfield, mei2019mean}.

\textbf{Contributions:} 
\begin{itemize}
    \item By employing algebraic properties Hermite polynomials, we derive explicit forms of the neural kernel and identify its fixed points, leading to many elegant results.  
    \item We demonstrate that the kernel sequence converges globally to a unique fixed point for a wide class of functions that cover all practically used activations, revealing inherent biases in deep representations. 
    \item We extend our analysis to networks with normalization layers and residual connections, highlighting their impact on the convergence behavior.
\end{itemize}

Understanding these dynamics contributes to a deeper understanding of how depth and nonlinearity interact with one another at initialization.

\section{Related works}
The study of deep neural networks through the lens of kernel methods and mean-field theory has garnered significant interest in recent years. The Neural Tangent Kernel (NTK) introduced by \citet{jacot2018neural} provided a framework to analyze the training dynamics of infinitely wide neural networks using kernel methods. This perspective was further expanded by \citet{lee2019wide} and \citet{arora2019exact}, who explored the connections between neural networks and Gaussian processes.

The propagation of signals in deep networks has been studied in the works of \citet{schoenholz2016deep} and \citet{pennington2017resurrecting}. Previous studies have also explored the critical role of activation functions in maintaining signal propagation and stable gradient behavior in deep networks \citep{hayou2019impact}. These studies focused on understanding the conditions required for the stable propagation of information and the avoidance of signal amplification or attenuation. However, these analyses often concentrated on local behaviors or specific conditions, leaving a gap in understanding the global evolution of representations across layers. 

Hermite polynomials have been used in probability theory and statistics, particularly in the context of Gaussian processes \citep{williams2006gaussian}. Although~\citet{poole2016exponential} and \citet{daniely2016toward} have utilized polynomial expansions to analyze neural networks, they do not study global dynamics in neural networks, as presented in this paper. To the best of our knowledge, the only existing work that uses Hermite polynomials in the mean-field regime to study the global dynamics of the kernel is by~\citet{joudaki2023impact}. However, this study only covers centered activations, which fail to capture several striking global dynamics covered in this study.

Our work extends these foundational studies by providing an explicit algebraic framework to analyze the global convergence of kernel sequences in deep networks with nonlinear activations. Using Hermite polynomial expansions, we offer a precise characterization of the kernel map and its fixed points, contributing new insights into the implicit biases of nonlinear activations.

\section{Preliminaries}
This section introduces the fundamental concepts, notations, and definitions used throughout this paper. 
If $X$ and $Y$ are vectors in $\R^n$, we use the inner product notation $\langle X, Y \rangle_\avg$ to denote their average inner product $ = \frac{1}{n} \sum_{i=1}^n X_i Y_i.$  We consider a feedforward neural network with $L$ layers and constant layer width $d$. The network takes an input vector $x \in \R^d$ and maps it to an output vector $h^L(x) \in \R^d$ through a series of transformations. The hidden representations in each layer $\ell$ are denoted by $h^\ell(x)$. The transformation in each layer is composed of a linear transformation followed by a nonlinear activation function $\phi$. We consider the multilayer perceptron (MLP), the hidden representation at layer $\ell$ is given by:
\begin{align}
& h^\ell(x) = \phi\left(\frac{1}{\sqrt{d}}W^\ell h^{\ell-1}(x)\right), && W^\ell \in \R^{d \times d},
\end{align}
where $h^0(x)=x$ is identified with the input. We can assume that elements of $W^\ell$ are drawn i.i.d.~from a zero mean distribution, unit variance distribution. Two prominent choices for weight distributions, Gaussian $N(0,1)$ and uniform distribution $\text{Uniform}[-1,1],$ satisfy these conditions. In some variations of the MLP, we will use normalization layers and residual connections. Finally, the neural kernel between two inputs $x$ and $y$ at layer $\ell$ is defined as:
\begin{align}
    \rho_\ell = \langle h^\ell(x), h^\ell(y) \rangle_\avg \,.
\end{align}

The main goal of this paper is to analyze the sequence $\{\rho_\ell\}$ at random initialization. The motivation for this analysis is to understand if architectural choices, specifically the activation function and the model depth, lead to specific biases of the kernel towards a certain fixed point. Namely, if there is a bias towards zero, it would imply that at initialization, the representations become more orthogonal. In contrast, a positive fixed point would mean the representations become more similar.

In the current setup, the sequence $\{\rho_\ell\}$ is a stochastic sequence or a Markov chain due to the random weights at each layer. However, we will show that under the mean-field regime, the sequence $\{\rho_\ell\}$ converges to a deterministic sequence, and we will analyze its properties.

\section{The mean-field regime}

In this section, we conduct a mean-field analysis of MLP to explore the neural kernel's fixed point behavior as the network depth increases. This approach allows us to gain insight into the global dynamics of neural networks, mainly how the similarity between two input samples evolves as they pass through successive network layers.
Now, we can state the mean-field regime for the kernel sequence, stating that in this regime, the sequence becomes deterministic. 

As will be proven later, this sequence's deterministic transition follows a scalar function defined below. 

\begin{definition}
    \label{def:kernel_map}
Given two random variables $X, Y$ with covariance $\rho,$ and activation $\phi,$ define the \emph{kernel map } $\kappa$ as the mapping between the covariance of preactivation and the covariance of post-activations:
\begin{align}
 \kappa(\rho):=\E\phi(X)\phi(Y), && 
 \begin{pmatrix}X \\ Y\end{pmatrix}\sim \mathcal N\left(0, \begin{pmatrix} 1 & \rho \\ \rho & 1 \end{pmatrix}
 \right).
 \label{eq:kernel_map}
\end{align}
\end{definition}

This definition has appeared in previous studies, namely the definition of dual kernel by~\citet{daniely2016toward}, and has been adopted by~\citet{joudaki2023impact} to study the effects of nonlinearity in neural networks. 

The following proposition formally states that in the mean-field regime, the kernel sequence follows iterative applications of the kernel map. 

\begin{proposition}
\label{prop:mean_field_kernel_general}
In the mean-field regime with ${d \to \infty}$, let $\rho_\ell$ denote the kernel sequence of an MLP with activation function $\phi$ obeying $\E\phi(X)^2=1$ for $X\sim \mathcal N(0,1)$. If each element of the weights drawn i.i.d.~from a distribution with zero mean and unit variance, the kernel sequence evolves deterministically according to the kernel map
\begin{align*}
&\rho_{\ell+1} = \kappa(\rho_\ell),
\end{align*}
where the initial value $\rho_0$ corresponds to the input, and $\kappa$ is defined in Definition~\ref{def:kernel_map}.

\end{proposition}

Historically, conditions similar to $\E \phi(X)^2=1$ have been used to prevent quantities on the forward pass from vanishing or exploding. For example, in a ReLU, half of the activations will be zeroed out, which will lead to a vanishing norm of forward representations. The initialization of
\citep{he2016deep} addresses that by scaling weights to maintain consistent forward pass norms across layers. This principle is further refined by \citet{klambauer2017self}, proposing the idea of self-normalizing activation functions, which ensure consistent mean and variances between pre- and post-activations. 

An interesting observation is that kernel dynamics is governed by kernel map $\kappa,$ defined on the basis of Gaussian preactivation, even if the weights are not Gaussian matrices. The main insight for this equivalence is that as long as elements are drawn i.i.d.~from a zero mean unit variance distribution, we can apply the Central Limit Theorem (CLT) to conclude that preactivations follow the Gaussian distribution. This simple yet elegant observation gives us a powerful analytic tool to study kernel dynamics by leveraging algebraic properties for Gaussian preactivations; we will automatically get the same results for a wide class of distributions. 

The proposition~\ref{prop:mean_field_kernel_general} tells us that the kernel sequence $\{\rho_\ell\}$ is precisely given by $\rho,\kappa(\rho),\kappa(\kappa(\rho)),\dots\,.$ Thus, we can analyze its convergence by studying the fixed points of the kernel map $\kappa,$
which are the values of $\rho^*$ that satisfy $\kappa(\rho^*) = \rho^*.$  With the assumption that $\E \phi(X)^2=1,$ the kernel map $\kappa$ is a mapping between $[-1,-1]$ to itself. Thus, Brower's fixed point theorem implies that the kernel map $\kappa$ has at least one fixed point $\rho^*.$ However, as we will show, there is potentially more than one fixed point, and it will be interesting to understand which ones are locally or globally attractive.

\section{Hermite expansion of activation functions}
Hermite polynomials possess completeness and orthogonality under the Gaussian measure. Therefore, any function that is square-integrable with respect to a Gaussian measure can be expressed as a linear combination of Hermite polynomials (see below). The square integrability rules out the possibility of having heavy-tailed postactivations without second moments. This holds for all activations that are used in practice. We use \emph{normalized} Hermite polynomials and their coefficients. 

\begin{definition}
Normalized Hermite polynomials $\he_k(x)$ are defined as follows
\begin{align*}
&\he_k(x) :=\frac{1}{\sqrt{k!}}(-1)^k e^{\frac{x^2}{2}} \frac{d^k}{dx^k} e^{-\frac{x^2}{2}}.
\end{align*}
\end{definition}

Although Hermite polynomials have been used in probability theory and statistics, particularly in the context of Gaussian processes \citep{williams2006gaussian}, their application in analyzing neural network dynamics provides a novel methodological tool. Previous works, such as \citet{daniely2016toward}, have utilized polynomial expansions to study neural networks, but our explicit use of Hermite polynomial expansions to derive the kernel map and analyze convergence is a new contribution.

The crucial property of Hermite polynomials is their  orthogonality:
\begin{align}\label{eq:hermite_orthogonality}
\E\he_k(X)\he_l(X) = \delta_{kl}, && X \sim \mathcal N(0,1).
\end{align}
Scaling by $1/\sqrt{k!}$ ensures that polynomials form a \emph{orthonormal} basis. Based on this property, we can define:
\begin{definition}\label{def:hermite_expansion}
Given a function $\phi$, square-integrable with respect to the Gaussian measure. Its Hermite expansion is given by:
\begin{align*}
&\phi = \sum_{k=0}^\infty c_k \he_k,&& c_k = \E\phi(X) \he_k(X), &&& X \sim \mathcal N (0,1)\,.
\end{align*}
Here, $c_k$ are called the Hermite coefficients of $\phi$. 
\end{definition}


In addition to orthogonality, Hermite polynomials have another ``magical'' property that is crucial for our later analysis. 

\begin{lemma}[Mehler's lemma]\label{lem:mehler_kernel}
For standard-normal random variables $X,Y$ with covariance $\rho$ it holds
\begin{align*}
\E\he_m(X)\he_n(Y) = \rho^n \delta_{mn}, && \begin{pmatrix}
    X \\ Y
\end{pmatrix}\sim \mathcal N\left(0, \begin{pmatrix} 1 & \rho \\ \rho & 1 \end{pmatrix}
 \right).
\end{align*}
\end{lemma}

This lemma states that given two Gaussian random variables $X, Y$ with covariance $\rho$, the expectation of the product of Hermite polynomials is zero unless the indices are equal (see explanations in Section~\ref{sec:proofs} for proof).
The lemma~\ref{lem:mehler_kernel} is crucial for our theory. Based on this lemma, we can express the kernel map $\kappa$ in terms of the Hermite coefficients, showing a particular structure of the kernel map.

\begin{corollary}
    \label{cor:hermite_covariance}
    \label{cor:kernel_map}
We have the following explicit form for the kernel map:
\begin{align*}
\kappa(\rho) = \sum_{k=0}^\infty c_k^2 \rho^k.
\end{align*}
\end{corollary}

The proof follows directly from the definition of kernel map, expanding $\phi$ in the Hermite basis and then applying Lemma~\ref{lem:mehler_kernel}. 

\citet{daniely2016toward} show a similar analytic form for dual activation, which aligns with our definition of kernel map. However, they did not use this to study global dynamics. To our knowledge, the only study that uses Hermite polynomials to study global dynamics is~\citep{joudaki2023impact}, which is limited to centered activations. i.e., those that satisfy $\E \phi(X)=0,$ or equivalently $\kappa(0)=c_0^2 = 0.$

\section{Convergence of the kernel with general activations}
In this section, we will show that for any nonlinear activation function, there is a unique fixed point $\rho^*$ that is globally attractive.

\begin{theorem}
\label{thm:global_attract}
Given a nonlinear activation function $\phi$ with kernel map $\kappa$ such that $\kappa(1)=1$. Define the kernel sequence $\rho_{\ell+1}=\kappa(\rho_\ell)$, with $\rho_0 \in (-1,1)$. Then there is a unique globally contracting fixed point $\rho^*$, which is necessarily non-negative $\rho^*\in[0,1]$. The only other fixed points distinct from $\rho^*,$ could be $\pm 1,$ neither of which is stable. Furthermore, we have the following contraction rate towards $\rho^*$:
\begin{enumerate}
    \item If $\kappa(0)=0$ then $\rho^*=0$ is an attracting with rate 
    \begin{align*}
    \frac{|\rho_\ell|}{1-|\rho_\ell|} \le \frac{|\rho_0|}{1-|\rho_0|} \alpha^\ell, && \alpha:=\frac{1}{2-\kappa'(0)}.
    \end{align*}
    \item If $\kappa(0)>0$ and $\kappa'(1)<1$, then $\rho^*=1$ is attracting, with rate 
    \begin{align*}
    |\rho_\ell-1| \le |\rho_0-1| \alpha ^\ell, && \alpha := \kappa'(1).
    \end{align*}
    \item If $\kappa(0) > 0$, and $\kappa'(1)=1$ then  $\rho^*=1$ is attracting with rate 
    \begin{align*}
    |\rho_\ell-1| \le \frac{|\rho_0-1|}{\ell\alpha|\rho_0-1|+1}, && \alpha = 1-\kappa(0)-\kappa'(0).
    \end{align*}
    \item If $\kappa(0) > 0$, and $\kappa'(1)>1$ then the attracting fixed point is necessarily in the range $\rho^*\in(0,1)$, satisfying $\kappa'(\rho^*) < 1,$ for which we have  
    \begin{align*}
    &|\rho_\ell-\rho^*| \le \frac{|\rho_0-\rho^*|}{1-|\rho_0|}\alpha^\ell  &&\alpha = \max\left\{1-\kappa(0),\kappa'(\rho^*),\frac{1-\rho^*}{2-\kappa'(\rho^*)}\right\},
    \end{align*}
    where $\alpha<1$.
\end{enumerate}
\end{theorem}

\paragraph{Implications.}  Let us take a step back and review the main takeaway of Theorem~\ref{thm:global_attract}. Omitting the constants and for sufficiently large depth $\ell,$ we have 
\begin{align*}
    \langle h^{\ell}(x), h^{\ell}(y)\rangle_\avg =  \begin{cases}
          0+ O(\alpha ^ \ell \langle x, y \rangle_\avg)  & \text{case 1}\\
         1 + O(\alpha ^ \ell \langle x, y \rangle_\avg)  & \text{case 2}\\
         1 + O \left(\langle x, y \rangle_\avg/\ell\alpha \right)   & \text{case 3}\\
         \rho^* + O(\alpha ^ \ell \langle x, y \rangle_\avg)  & \text{case 4},\\
    \end{cases}
\end{align*}
where $\langle x,y\rangle_\avg $ denotes the input similarly. 

Broadly speaking, we can think of three categories of bias:

\begin{itemize}
    \item \textit{Orthogonality bias $\rho^*=0$}: Implies that as the network becomes deeper, representations are pushed towards orthogonality exponentially fast. We can think of this case as a bias towards independence. 
    \item \textit{Weak similarity bias $\rho^*\in (0,1)$}: Implies that as the network becomes deeper, the representations form angles between $0$ and $\pi/2$. Thus, in this case, the representations are neither completely aligned nor completely independent. 
    \item \textit{Strong similarity $\rho^*=1$}: Implies that as the network becomes deeper, representations become more similar or aligned, as indicated by inner products converging to one, exponentially fast in case 2, polynomially in case 3. 
\end{itemize}

See section~\ref{sec:experiments} for a review of commonly used activations, and their biases according to Theorem~\ref{thm:global_attract}.

The bias of activation and normalization layers has been extensively studied in the literature. For example, \citet{daneshmand2021batch} show that batch normalization with linear activations makes representations more orthogonal, relying on a technical assumption that is left unproven (see assumption $\mathcal{A}_1$). Similarly, \citet{joudaki2023bridging} extend this to odd activations yet introduce another technical assumption about the ergodicity of the chain, which is hard to verify or prove (see Assumption 1). In a similar vein, ~\citet{meterez2024towards} show the global convergence towards orthogonality in a network with batch normalization and orthogonal random weights, but do not theoretically analyze nonlinear activations. \citet{yang2019meanfield} prove the global convergence of the Gram matrix of a network with linear activation and batch normalization (see Corollary F.3) toward a fixed point. However, the authors explain that because they cannot establish such a global convergence for general activations (see page 20, under the paragraph titled Main Technical Results), they resort to finding locally stable fixed points, meaning that if preactivations have such a Gram matrix structure and they are perturbed infinitesimally. Recent studies have also examined the evolution of covariance structures in deep and wide networks using stochastic differential equations \citep{li2022neural}, which is similar to kernel dynamics and kernel ODE introduced here. However, the covariance SDE approach does not theoretically show the global convergence towards these solutions. Finally, \citet{joudaki2023impact} establish a global bias towards orthogonality for centered activations (see Theorem A2), which aligns with case 1 of Theorem~\ref{thm:global_attract}, and to extend to other activations, they add layer normalization after the activation to make it centered. In contrast, Theorem~\ref{thm:global_attract} covers all existing activations. 

One alternative way of interpreting the results of Theorem~\ref{thm:global_attract} is that as we pass two inputs through nonlinear activations, it `forgets' the similarity between inputs and converges to a fixed point value that is independent of the inputs and only depends on the activation. Taking the view that the network ought to remember some similarity of the input for training, we can leverage Theorem~\ref{thm:global_attract} to strike a balance between depth and nonlinearity of the activations, as argued in the following.

\begin{corollary}
\label{cor:convergence_precision}
Let $\epsilon > 0$ be the some numerical precision, and let $\phi$ be an activation function such that $\E\phi(X)^2=1$, with its kernel map obeying $\kappa'(1) < 1$. Then, for any initial pair of inputs $x$ and $y$ after $\ell \ge L:= \frac{\ln(1/\epsilon)}{\ln(1/\kappa'(1))}$ layers, the representations of $x$ and $y$ will become numerically indistinguishable.
\end{corollary}

Approximately, $\ln(1/\epsilon)$ represents the number of bits of numerical precision available in floating-point representations. For example for \texttt{float32} numbers, $\ln(1/\epsilon)$ is approximately $\ln(1/2^{-128}) \approx 88$. Therefore, when the network depth $\ell$ exceeds threshold $88/\ln(1/\kappa'(1))$, the representations of any two inputs will become numerically indistinguishable from one another. This effect propagates through subsequent layers, potentially making training ineffective or impossible. Notably, ReLU and sigmoid satisfy the conditions of this corollary, implying that this phenomenon can occur in networks utilizing these activations. Previous works on rank collapse~\citep{daneshmand2020batch,noci2022signal} and covariance degeneracy~\citep{li2022neural} have observed and studied this phenomenon in various settings. However, the explicit and closed-form relationship presented in Corollary~\ref{cor:convergence_precision} provides a novel and precise quantification of this effect.

One of the most striking results of this theorem is that for any nonlinear activation, there is exactly one fixed point $\rho^*$ that is globally attractive, while other fixed points are not stable. Furthermore, the fixed point is necessarily non-negative. This implies that for any MLP with nonlinear activations, the covariance of preactivations will converge towards a fixed point, which is always non-negative. 
It turns out that there is a geometric reason why no globally attracting negative fixed point could exist (See remark~\ref{rem:no_negative_geometric})

Theorem~\ref{thm:global_attract} has several aspects that may seem counterintuitive. For instance, case 1 demonstrates that a quantity $|\rho|/(1-|\rho|)$ decays exponentially as depth increases. To gain a better understanding of the theories behind Theorem~\ref{thm:global_attract}, we can cast our layer-wise kernel dynamics as a continuous-time dynamical system, which is discussed next. For a detailed proof of Theorem~\ref{thm:global_attract}, please refer to Section~\ref{sec:proofs}.

\subsection{A continuous time differential equation view to the kernel dynamics} 
One of the most helpful ways to find insights into discrete fixed point iterations is to cast them as a continuous problem. More specifically, consider our fixed point iteration: 
\begin{align*}
    \Delta \rho_\ell &= \rho_{\ell+1} - \rho_\ell = - \rho_\ell + \sum_{k=0}^\infty c_k^2 \rho^k_\ell.
\end{align*}
We can replace this discrete iteration with a continuous time differential equation, which we will refer to as the kernel ODE:
\begin{align}\tag{kernel ODE}\label{eq:kernel_ODE}
    d\rho/dt = -\rho + \sum_{k=0}^\infty c_k^2 \rho^k.
\end{align}

Recent studies have introduced stochastic models such as the Neural Covariance SDE by~\citet{li2022neural}, which can be viewed as the stochastic analog of kernel ODE. However, kernel ODE will capture the most important parts of the discrete kernel dynamics for our main purpose of characterizing the attracting fixed points. The continuous analog of fixed points is a $\rho^*$ that satisfies $\rho'(\rho^*) = 0.$ Furthermore, the fixed point is globally attracting if we have $\rho'$ become negative for $\rho>\rho^*$ and positive for $\rho<\rho^*.$ We can say that the fixed point is locally attractive if this property only holds in some small neighborhood of $\rho^*.$

While the kernel ODE presents an interesting transformation of the problem, it does not necessarily have a closed-form solution in its general form. However, our main strategy is to find worst-case (slowest) scenarios for convergence that correspond to tractable ODEs. In many cases, however, the worst-case scenario depends on some boundary conditions. Thus, we can design worst-case ODEs for different cases and combine them in a joint bound. Let us use our centered activation equation as a case study. 

\subsection{Kernel ODE for centered activations}
For the case of centered activation $\E \phi(X)=0,$ corresponding to case 1 of Theorem~\ref{thm:global_attract}, the kernel ODE is given by
\begin{align*}
    d\rho/dt = -(1-c_1^2)\rho + \sum_{k=2}^\infty c_k^2\rho^k\,,
\end{align*}
where the first term $k=0$ is canceled due to the assumption $\kappa(0)=c_0^2 = 1.$ 

Observe that $\rho'<0$ when $\rho>0,$ and $\rho'>0$ when $\rho<0.$ This implies that $\rho(t) \to 0$ for sufficiently large $t.$ Intuitively, the terms in $\rho'$ that have the opposite sign of $\rho,$ contribute to a faster convergence, while terms with a similar sign contribute to a slower convergence. Thus, we can ask what distribution of Hermite coefficients $\{c_k\}$ corresponds to the worst case, slowest convergence. If this is a tractable ODE, we can use its solution to bound the convergence of the original ODE. It turns out that the worst case depends on the positivity of $\rho,$ which is why we study them separately.

\textit{Positive range.}
Let us first consider the dynamics when $\rho\ge 0.$ 
In this case, the term corresponding to $k=1$ contributes positively to a faster convergence, while terms $k\ge 2$ make convergence slower. In light of this observation, a worst-case (slowest possible) convergence rate happens when the positive terms are maximized, which occurs when the weight of all $\sum_{k=2}^\infty c_k^2 $ is concentrated on the $k=2$ term, leading to kernel ODE
\begin{align*}
    d\rho/dt &=  -(1-c_1^2)\rho + (1-c_1^2) \rho^2,
\end{align*}
where we used the fact that $\sum_{k=1}^\infty c_k^2 = 1.$ 
Finally, we can solve this ODE 
\begin{align*}
    \int \frac{d\rho}{\rho(1-\rho)} = - (1-c_1^2)\int t 
    \implies \frac{|\rho(t)|}{|1-\rho(t)|} = C \exp(-t(1-c_1^2)),
\end{align*}
where $C$ corresponds to the initial values.

\textit{Negative range.}
Now, let us assume $\rho<0.$ In this case, only the odd terms in $\sum_{k=2}^\infty c_k^2 \rho^k$ contribute to a slowdown in convergence. Thus, the worst case occurs when all the weight of coefficients is concentrated in $k=3,$ leading to the kernel ODE:
\begin{align*}
d\rho/dt = -(1-c_1^2)\rho + (1-c_1^2) \rho^3
\implies \frac{|\rho(t)|}{\sqrt{1-\rho^2(t)}} = C' \exp(-t(1-c_1^2)),
\end{align*}
where $C'$ corresponds to the initial values. 

To summarize, we have obtained:
\begin{align*}
    \begin{cases}
        \frac{|\rho(t)|}{1-\rho(t)} \le C \exp(-\ell(1-c_1^2)) & \rho_0 \in \R^+\\
        \frac{|\rho(t)|}{\sqrt{1-\rho(t)^2}}\le C' \exp(-\ell(1-c_1^2)) & \rho_0\in\R^-.
    \end{cases}
\end{align*}
Now, we can use a numerical inequality $\sqrt{1-x^2} \ge 1-|x|$ valid for all $x\in(-1,1),$ and by solving for the constant, we can construct a joint bound for both cases:
\begin{align*}
    \frac{|\rho(t)|}{1-|\rho(t)|} \le \frac{|\rho_0|}{1-|\rho_0|} \exp(-t (1-\kappa'(0))), 
\end{align*}
where we have also replaced $c_1^2=\kappa'(1).$  

\paragraph{Key insights.}
The kernel ODE perspective reveals two important insights. 
First, the appearance of $|\rho|/(1-|\rho|)$ formula in our bound is due to the fact that $x/(1-x) = \exp( -c t) $ is a fundamental solution of the differential equation $x' = - c x(1-x).$ Second, it reveals the importance of the positivity of coefficients in $\sum_{k=0}^\infty c_k^2\rho^k,$ which allowed us to arrive at a worst-case pattern for the coefficients. 
This further highlights the value of algebraic properties of Hermite polynomials that canceled out the cross terms in $\kappa.$

We can observe that the kernel ODE rate derived by analysis of kernel ODE is largely aligned with the exponential convergence of this term in Theorem~\ref{thm:global_attract}. The slight discrepancy between the exponential rates between Theorem~\ref{thm:global_attract} and the kernel ODE rate $\exp(-(1-\kappa'(0))$ is due to the discrete-continuous approximation. Despite this small discrepancy, the solution to kernel ODE can help us arrive at the right form for the solution. We can leverage the solution to kernel ODE to arrive at a bound for the discrete problem, namely using induction over steps.

Since Theorem~\ref{thm:global_attract} only provides an upper bound and not a lower bound for convergence, one natural question is: how big is the gap between this worst case and the exact ODE convergence rates? We can construct activations where the convergence is substantially faster, such as doubly exponential convergence (see Corollary~\ref{cor:double_exp}). However, for all practically used activations, the gap between our bound and the real convergence is enough (See section~\ref{sec:experiments}).

\section{Extension to normalization layers and residual connections}

In this section, we extend our analysis to MLPs that incorporate normalization layers and residual (skip) connections and examine how they affect convergence.

\subsection{Residual connections}

Residual connections, inspired by ResNets~\citep{he2016deep}, help mitigate the vanishing gradient problem in deep networks by allowing gradients to flow directly through skip connections. We consider the MLP with residual strength $t$ given by
\begin{align*}
h^\ell = \sqrt{1-r^2} \, \phi\left( \frac{W^\ell}{\sqrt{d}}  h^{\ell-1} \right) + r\, \frac{P^\ell}{\sqrt{d}}  h^{\ell-1},
\end{align*}
where $W^\ell$ and $P^\ell$ are independent weight matrices of dimension $d \times d$ with entries drawn i.i.d.~from a zero-mean, unit-variance distribution, and $r \in (0, 1)$ modulates the strength of the residual connections (the bigger $r$, the stronger residuals will be).

\begin{proposition}
\label{prop:residual_kernel_map}
For an MLP with activation $\phi$ satisfying $\E\phi(X)^2,\,X\sim N(0,1),$ and residual parameter $r,$ we have the residual kernel map
\begin{equation}
\kappa_\res(\rho) = (1-r^2) \kappa_\phi(\rho) + r^2 \rho,
\end{equation}
where $\kappa_\phi$ denotes kernel map of $\phi.$

Furthermore, we have: 1) fixed points of $\kappa_\res(\rho)$ are the same as those of $\kappa_\phi(\rho),$ 2)  $\rho^*$ is a globally attracting fixed point of  $\kappa_\phi$ is a globally attracting fixed point of $\kappa_\psi,$ 3) the convergence of residual kernel map $\kappa_\psi$ is monotonically decreasing in $r$ (the stronger residuals, the slower convergence). 
\end{proposition}


\paragraph{Implications.} Proposition~\ref{prop:residual_kernel_map} reveals that residual connections modify the kernel map by blending the original kernel map with the identity function, weighted by the residual $r$. This adjustment has several interesting implications. Our analysis here gives a quantitative way of balancing depth with the nonlinearity of activations. For example, \citet{li2022neural} show as we `shape' the negative slope of a leaky ReLU towards identity, it will prevent degenerate covariance in depth. In light of proposition~\ref{prop:residual_kernel_map} and corollary~\ref{cor:convergence_precision}, we can write leaky ReLU as a linear combination of ReLU and identity and conclude that as residual strength increases ($r\to 1$), the convergence rate becomes slower and thus degeneracy happens at deeper layers, which aligns with the prior study. 
Another interesting byproduct of proposition~\ref{prop:residual_kernel_map} is that it shows that when $r\to 1$ (highly strong residuals), the kernel ODE becomes an exact model for kernel dynamics (See remark~\ref{rem:residual_ODE}).


\subsection{Normalization layers}

Normalization layers are widely used in deep learning to stabilize and accelerate training. In this section, we focus on two common normalization layers:
\begin{itemize}
    \item \textit{Layer Normalization (LN)} \citep{ba2016layer}:
  \begin{equation}
  \operatorname{LN}(z) = \frac{z - \mu(z)}{\sigma(z)},
  \end{equation}
  where $\mu(z)$ is the mean, and $\sigma(z)$ is the standard deviation of the vector $z \in \R^d$.
  \item \textit{Root Mean Square (RMS) Normalization}:
  \begin{equation}
  \operatorname{RN}(z) = \frac{z}{ \sqrt{ \frac{1}{d} \sum_{i=1}^d z_i^2 } }.
  \end{equation}
\end{itemize}

The following theorem characterizes the joint activation and normalization kernel.

\begin{proposition}
\label{prop:normalization_kernel_map}
Let us assume $\phi$ satisfies $\E \phi(X)^2 = 1$ where $X\sim N(0,1).$ Consider an MLP layers 
\begin{align*}
    h^\ell = \psi\left( \frac{1}{\sqrt{d}}W^\ell h^{\ell-1}\right),
    && \psi \in \{\phi\circ RN,\phi\circ LN, RN\circ \phi, LN\circ \phi\},
\end{align*}
where $\psi$ denotes the joint activation and normalization layers. 
We have the following joint normalization and activation kernel:
\begin{align*}
    \kappa_\psi(\rho) = \begin{cases}
         \frac{\kappa_\phi(\rho)-\kappa_\phi(0)}{1-\kappa_\phi(0)} & \text{if } \psi=LN\circ \phi\\
        \kappa_\phi(\rho) & \text{otherwise},\\
    \end{cases}
\end{align*}
where $\kappa_\phi$ denotes the kernel map of $\phi$. 
\end{proposition}


\paragraph{Implications.}
Proposition~\ref{prop:normalization_kernel_map} implies that normalization layers adjust the kernel map by scaling and, in some cases, shifting it. When layer normalization is applied after activation, the activation and kernel map become centered. This means that for activations such as ReLU or sigmoid, the fixed point of the kernel sequence changes from $\rho^*=1$ to $\rho^*=0$, indicating a shift from a bias towards the orthogonality of representations to aligning them. This result supports previous findings by~\citet{joudaki2023impact} regarding the effect of layer normalization in centering, and extends it to other configurations (order of normalization and activation) for different orders of layers, as well as using root mean normalization. Furthermore, our findings align with studies showing that batch normalization induces orthogonal representations in deep networks \citep{yang2019meanfield,daneshmand2021batch}.

\section{Discussion}
The power of deep neural networks fundamentally arises from their depth and nonlinearity. Despite substantial empirical evidence demonstrating the importance of deep, nonlinear models, a rigorous theoretical understanding of how they operate and learn remains a significant theoretical challenge. Much of classical mathematical and statistical machinery was developed for linear systems or shallow models \citep{hastie2009elements}, and extending these tools to deep, nonlinear architectures poses substantial difficulties. Our work takes a natural first step into this theoretical mystery by providing a rigorous analysis of feedforward networks with nonlinear activations. 

We showed that viewing activations in the Hermite basis uncovers several strikingly elegant properties of activations. Many architectural choices, such as normalization layers, initialization techniques, and CLT-type convergences, make Gaussian preactivations central to understanding the role of activations. Hermite expansion can be viewed as the Fourier analog for the Gaussian kernel. These facts and our theoretical results, suggest that viewing activations in the Hermite basis is more natural than in the raw signal, analogous to viewing convolutions in the Fourier basis. For example, the fact that the kernel map $\kappa$ is analytic and has a positive power series expansion, i.e., is infinitely smooth, does not require the activation $\phi$ to be smooth or even continuous. Thus, as opposed to many existing analyses that consider smooth and non-smooth cases separately, our theoretical framework gives a unified perspective. As an interesting example, smoothness, which has been observed to facilitate better training dynamics~3\citep{hayou2019impact}, appears as a faster decay in Hermite coefficients. Thus, similar to leveraging the Fourier transform for understanding and designing filters, one can aspire to use Hermite analysis for designing and analyzing activation functions.

One of the key contributions of our work is the global perspective in analyzing neural networks. Traditional analyses often focus on local structure, such as Jacobian eigenspectrum by~\citet{pennington2018emergence}. While these local studies provide valuable insights, they do not capture the global kernel dynamics. More concretely, real-world inputs to neural networks, such as images of a dog and a cat, are not close enough to each other to be captured in the local changes. Our global approach allows us to study the evolution of representations across layers for any pair of inputs, providing further insights into how deep networks transform inputs with varying degrees of similarity.

One of the central assumptions for our approach is operating in the mean-field, i.e., infinite-width regime. While this approach led us to establish numerous elegant properties, it would be worth exploring the differences between finite-width and infinite-width results. This remains an important avenue for future work. Similarly, in the CLT application of Proposition~\ref{prop:mean_field_kernel_general}, the exact discrepancies between the limited width and the infinite width would deepen our understanding of practical neural networks. Another intriguing endeavor is the potential to extend our framework to develop a theory for self-normalizing activation functions in a vein similar to~\citep{klambauer2017self}. Finally, building on the mathematical foundations laid here to study first- and second-order gradients and their impact on training dynamics remains a highly interesting topic for future research.



\bibliographystyle{plainnat}
\bibliography{refs}

\appendix
\renewcommand{\thetable}{S.\arabic{table}}
\renewcommand{\theremark}{S.\arabic{remark}}
\renewcommand{\thefigure}{S.\arabic{figure}}
\renewcommand{\thetheorem}{S.\arabic{theorem}}
\renewcommand{\thecorollary}{S.\arabic{corollary}}
\renewcommand{\thecorollary}{S.\arabic{lemma}}

\section{Supplementary theorems and proofs}\label{sec:proofs}
This section is dedicated to the proof of main theoretical results presented in the paper. 

Theorem~\ref{thm:global_attract_centered}, which is closely aligned with Theorem A2 of~\citet{joudaki2023impact}, and for historical reasons, it is kept here, while Theorem~\ref{thm:global_attract}, which is the general version, is citing Theorem~\ref{thm:global_attract_centered} as one specific case. Likewise, Lemma~\ref{lem:mehler_kernel}'s statement is a close replica of Lemma A8 in~\cite{joudaki2023impact}. Thus, in both cases, we have skipped the proofs here, and refer the interested reader to~\cite{joudaki2023impact} for the proofs.

\begin{proof}[Proof of Proposition~\ref{prop:mean_field_kernel_general}]
 The proof is a straightforward application of the law of large numbers, as the mean-field regime implies that the sample means converge to the population means. Let us inductively assume that at layer $\ell$, it holds $\frac1d\|h^{\ell}(x)\|^2 = 1$ and $\frac1d\|h^{\ell}(y)\|^2=1,$ and $\frac1d\langle h^{\ell}(x),h^{\ell}(y)\rangle = \rho_{\ell}$. Thus, defining $a = W^{\ell+1} h^\ell(x)$ and $b = W^{\ell+1} h^\ell(y),$ their elements $a_i$'s and $b_i$'s follow $N(0,1)$, and have covariance $\E a_i b_i = \rho_{\ell}.$ Thus, again by construction we have $\E \phi(a_i)^2 = \E \phi(b_i)^2 = 1$, and $\E \phi(a_i)\phi(b_i) = \rho_{\ell+1}.$ Finally, because $a_i$'s and $b_i$'s are~i.i.d. copies of their respective distributions, based on the law of large numbers, we can conclude the samples means will converge to their expectations:
 \begin{align*}
&\frac{1}{d} \|h^{\ell+1}(x)\|^2 = \frac{1}{d} \|\phi(a)\|^2 = \frac{1}{d} \sum_{i=1}^d \phi(a_i)^2 \xrightarrow{d \to \infty} \E \phi(a_1) = 1\\
&\frac1d \|h^{\ell+1}(y)\|^2 = \frac1d \|\phi(b)\|^2 = \frac1d \sum_{i=1}^d \phi(b_i)^2 \xrightarrow{d \to \infty}  \E \phi(b_1) =  1\\
&\frac1d \langle h^{\ell+1}(x),h^{\ell+1}(y)\rangle = \frac1d \langle \phi(a),\phi(b)\rangle = \frac1d \sum_{i=1}^d \phi(a_i)\phi(b_i) \xrightarrow{d \to \infty}  \E \phi(a_1)\phi(b_1) = \rho_{\ell+1}
 \end{align*}
Thus, we have proved the induction hypothesis for step $\ell+1$. This concludes the proof. 
\end{proof}

\begin{theorem}\label{thm:global_attract_centered}
Let $\phi$ be a nonlinear activation with kernel map $\kappa$ that is centered $\kappa(0)=0,$ and $\kappa(1)=1$. 
Let $\rho_\ell$ be the kernel sequence generated by $\kappa$, given the initial value $\rho_0 \in(-1 , 1)$.
Then, the sequence contracts towards the fixed point at zero $\rho^*=0$ with rate $\alpha$:
    \begin{align}
        &\frac{|\rho_\ell|}{1-|\rho_\ell|} \le \frac{|\rho_0|}{1-|\rho_0|}\alpha^{\ell}, && \alpha := \frac{1}{2-\kappa'(0)},
    \end{align}
 where $\alpha<1.$ The only other fixed points can be $\rho^*=\pm 1$, and none of them is locally or globally attractive.
\end{theorem}

\begin{proof}[Proof of Theorem~\ref{thm:global_attract}]
We will proof all cases individually, starting with the first case that falls directly under a previous theorem. 

\subsection*{Case 1: $\kappa(0)=0$}
We can observe that the case where $\kappa(0)=0,$ falls directly under Theorem~\ref{thm:global_attract}, and thus there is no need to prove it again.

\subsection*{Cases 2,3: $\kappa(0)>0$ and $\kappa'(1)\le 1$} In this part, we jointly consider two cases where $\kappa(0)>0$ and $\kappa'(1) < 1,$ and $\kappa'(1)=1.$ Let us consider the ratio between distances $|\rho_\ell-1|$:
\begin{align*}
    \frac{|\kappa(\rho)-1|}{|\rho-1|} &= \frac{1-\kappa(\rho)}{1-\rho} \\
    &=\frac{\kappa(1)-\kappa(\rho)}{1-\rho}\\
    &=\frac{\sum_{k=1}^\infty c_k^2 (1-\rho^k)}{1-\rho}\\
    &=\sum_{k=1}^\infty c_k^2 \sum_{i=0}^{k-1}\rho^i\\
\implies \frac{|\kappa(\rho)-1|}{|\rho-1|} &=\kappa'(1)-\kappa'(1)+\sum_{k=1}^\infty c_k^2 \sum_{i=0}^{k-1} \rho^i\\
&= \kappa'(1)-\sum_{k=1}^\infty c_k^2 \big(k-\sum_{i=0}^{k-1} \rho^i\big)\\
\end{align*}
Clearly, the term $k-\sum_{i=0}^{k-1} \rho^i$ is is always non-negative. Thus, if $\kappa'(1)<1,$ then we have the contraction 
\begin{align*}
    \frac{|\kappa(\rho)-1|}{|\rho-1|} \le \kappa'(1) \implies |\rho_\ell-1| \le |\rho_0-1| \kappa'(1)^\ell.
\end{align*}
Otherwise, if $\kappa'(1)=1,$ we have 
\begin{align*}
    \frac{|\kappa(\rho)-1|}{|\rho-1|} = 1-\sum_{k=1}^\infty c_k^2 \big(k-\sum_{i=0}^{k-1} \rho^i\big).
\end{align*}
Now, observe that the first term for $k=1$ is zero. Furthermore, the sequence $k-\sum_{i=0}^{k-1}\rho^i$ is monotonically increasing in $k$. Thus, the smallest value the weighted sum can achieve is if all of the weights of terms above $k\ge 2$ are concentrated in $k=2,$ which leads to the contraction
\begin{align*}
    \frac{|\kappa(\rho)-1|}{|\rho-1|} \le 1-(1-c_0^2-c_1^2)(2-1-\rho) \\
    = 1- (1-c_0^2-c_1^2) (1-\rho).
\end{align*}
Now, define sequence $x_\ell:= 1-\rho_\ell,$ and observe that we have 
\begin{align*}
    x_{\ell+1} \le x_\ell (1-\alpha x_\ell), && \alpha = 1-c_0^2-c_1^2,
\end{align*}
where $\alpha > 0$ if the activation is nonlinear. 
We can prove inductively that 
\begin{align*}
x_\ell\le \frac{x_0}{\ell\alpha x_0 + 1}.
\end{align*}
If we plug in the definition of $x_n$ we have proven
\begin{align*}
|\rho_\ell-1| \le \frac{|\rho_0-1|}{\ell \alpha |\rho_0-1| + 1}.
\end{align*}

\subsection*{Case 4: $\kappa(0)>0$ and $\kappa'(1)>1$}
The main strategy is to prove some contraction of $\kappa(\rho)$ towards $\rho^*$, under the kernel map $\kappa$. In other words, we need to show $|\kappa(\rho)-\rho^*|$ is smaller than $|\rho-\rho^*|$ under some potential. First, we assume there is a $\rho^*$ such that $\kappa'(\rho^*)<1,$ and show this contraction, and later prove its existence and uniqeness. 

To prove contraction towards $\rho^*$ when $\kappa'(\rho^*)<1$, we consider three cases: 1) If $\rho > \rho^*$, 2) If $\rho \in [0,\rho^*]$, and 3) If $\rho < 0$. However, the bounds will be of different potential forms and will have to be combined later. Let $\kappa(\rho) = \sum_{k=0}^\infty c_k^2 \rho^k$ be the kernel map with $\kappa(1)= 1$ with fixed point $\rho^*$ that satisfies $\kappa'(\rho^*)<1.$

\begin{itemize}
\item \textbf{$\rho\ge \rho^*$.} we will prove:
\begin{align*}
\frac{|\kappa(\rho)-\rho^*|}{1-\kappa(\rho)} \le \frac{|\rho-\rho^*|}{1-\rho} \kappa'(\rho^*)
\end{align*}

We have the series expansion around $\rho^*$: $\kappa(\rho) = \rho^* + \sum_{k=1}^\infty a_k (\rho-\rho^*)^k$. For points $\rho\ge \rho^*$, we will have $\kappa(\rho)\ge \rho^*$, thus we can write

\begin{align*}
&\frac{\kappa(\rho)-\rho^*}{1-\kappa(\rho)} \\
&= \frac{\sum_{k=1}^\infty a_k (\rho-\rho^*)^k}{\kappa(1)- \kappa(\rho^*)}\\
&= \frac{\sum_{k=1}^\infty a_k (\rho-\rho^*)^k}{\sum_{k=1}^\infty a_k (1-\rho^*)^k - \sum_{k=1}^\infty a_k (\rho-\rho^*)^k}  \\
&= \frac{\rho-\rho^*}{1-\rho}\cdot\frac{\sum_{k=1}^\infty a_k (\rho-\rho^*)^{k-1}}{\sum_{k=1}^\infty a_k (\sum_{i=0}^{k-1}(1-\rho^*)^i(\rho-\rho^*)^{k-1-i})}\\
&\le \frac{\rho-\rho^*}{1-\rho}\frac{\sum_{k=1}^\infty a_k (\rho-\rho^*)^{k-1}}{\sum_{k=2}^\infty a_k (1-\rho^*)^{k-1} + \sum_{k=1}^\infty a_k (\rho-\rho^*)^{k-1}},
\end{align*}
where in the last line the inequality is due to the fact that we are only retaining the terms corresponding to $i=0$ and $i=k-1$ in the denominator.
Now, note the right hand side maximizes when $\rho\in[\rho^*,1]$ is maximized, which is obtained when $\rho=1$:
\begin{align*}
&\frac{\kappa(\rho)-\rho^*}{1-\kappa(\rho)}\\
&\le \frac{\rho-\rho^*}{1-\rho}\frac{\sum_{k=1}^\infty a_k (1-\rho^*)^{k-1}}{\sum_{k=2}^\infty a_k (1-\rho^*)^{k-1} + \sum_{k=1}^\infty a_k (1-\rho^*)^{k-1}}.
\end{align*}
Now we can observe that the numerator and denominator correspond to 
\begin{align*}
    &=\frac{\rho-\rho^*}{1-\rho}\frac{\kappa(1)-\rho^*}{2\kappa(1)-\kappa'(\rho^*)}\\
&=\frac{\rho-\rho^*}{1-\rho}\frac{1-\rho^*}{2-\kappa'(\rho^*)}
\end{align*}
Thus, we have proven that 
\begin{align*}
\rho \ge \rho^* \implies \frac{|\kappa(\rho)-\rho^*|}{1-\kappa(\rho)} \le \frac{|\rho-\rho^*|}{1-\rho} \frac{1-\rho^*}{2-\kappa'(\rho^*)}
\end{align*}


\item \textbf{$0\le \rho \le \rho^*$.}
Consider $\rho \in [0,\rho^*]$. For these $\kappa'(\rho)$ is always monotonically increasing, implying that $\kappa'(\rho)\le \kappa'(\rho^*) < 1$. Thus, $|\kappa(\rho)-\rho^*| \le \kappa'(\rho^*) |\rho - \rho^*|$. Thus, by Banach fixed point theorem, we have that $\kappa$ is a contraction in this range with a rate $\kappa'(\rho^*)$:
\begin{align*}
0\le \rho \le \rho^* \implies |\kappa(\rho) - \rho^*| \le \kappa'(\rho^*) |\rho-\rho^*|
\end{align*}

\item \textbf{$-1 \le \rho \le 0$.}
Finally, let us consider $\rho \le 0$. 
Recall that we have $\kappa(1) = 1$. Thus, we can express $\kappa(\rho)-1$ as product of $(\rho-1)$ with some power series $q(\rho)$:
\begin{align*}
\kappa(\rho) - 1 = (\rho-1)q(\rho), \quad q(\rho) = \sum_{k=0}^\infty b_k \rho^k.
\end{align*}
In fact, we can expand $\kappa(\rho)$ in terms of these new coefficients
\begin{align*}
\kappa(\rho) = 1-b_0 + \sum_{k=0}^\infty (b_k - b_{k+1}) \rho^k
\end{align*}
Due to the non-negativity of coefficients of $\kappa$, we can conclude $1\ge b_0 \ge b_1 \ge ...$. Based on this observation, for $0 < \rho < 1$, we can conclude
$q(-\rho) = b_0 - b_1 \rho + b_2 \rho^2 - ... \le b_0$. Because we can pair each odd and even term $-b_{k}\rho^k + b_{k+1} \rho^{k+1}$ for all odd $k$, and because coefficients $b_k \ge b_{k+1}$ and $\rho^{k} \ge \rho^{k+1}$ for $\rho \in [0,1]$, we can argue $q(-\rho) \le b_0 = 1 - c_0^2$. Now, plugging this value into the kernel map for $0 < \rho < 1$, we have:
\begin{align*}
\kappa(-\rho) &= 1 - (1+\rho)q(-\rho) \\
&\ge 1-(1+\rho)(1-c_0^2)\\
&= 1 - 1 - \rho + c_0^2 (1+\rho)\\
&\implies \kappa(-\rho) + \rho \ge c_0^2 (1+\rho)\\
&\implies -\rho + c_0^2(1+\rho) \le \kappa(-\rho) \le \kappa(\rho)
\end{align*}
Now, if we assume $\kappa(-\rho) \le \rho^*$ then
\begin{align*}
\frac{|\kappa(-\rho)-\rho^*|}{|-\rho-\rho^*|}&=\frac{\rho^* - \kappa(-\rho)}{\rho^* + \rho} \\
&\le \frac{\rho^* + \rho - c_0^2 (1+\rho)}{\rho+\rho^*} \\
&= 1- \frac{c_0^2(1+\rho)}{\rho+\rho^*} \\
&\le 1-c_0^2\\
&= 1 - \kappa(0).
\end{align*}
Now, if we assume $\kappa(-\rho) \ge \rho^*$, knowing that $\kappa(-\rho)\le \kappa(\rho)$, which necessitates $\rho\ge \rho^*$, which implies $\kappa(\rho)\le \rho$. Thus, we have
\begin{align*}
\frac{|\kappa(-\rho)-\rho^*|}{|-\rho-\rho^*|} &= \frac{\kappa(-\rho)-\rho^*}{\rho+\rho^*}\\
&\le \frac{\kappa(-\rho)-\rho^*}{\rho+\rho^*} \\
&\le \frac{\rho-\rho^*}{\rho+\rho^*} \\
&\le \frac{1-\rho^*}{1+\rho^*} \\
&\le 1-\rho^*
\end{align*} 
Combining both cases we have 
\begin{align*}
\rho \le 0 \implies \frac{|\kappa(\rho)-\rho^*|}{|\rho-\rho^*|} &\le 1-\min(\kappa(0),\rho^*).
\end{align*}
We can further prove that $\rho^* = \kappa(\rho^*) = k(0) + \text{non-negative terms}$, which implies that $\rho^* \ge k(0)$. Thus, we can conclude that 
\begin{align*}
\rho \le 0 \implies |\kappa(\rho)-\rho^*| \le (1 - k(0))|\rho-\rho^*|
\end{align*}
\end{itemize}

\textbf{Combining the cases}
Let us summarize the results so far. We have proven the existence of a unique fixed point $\rho^*\in[0,1]$ such that $\kappa'(\rho^*)< 1,$ and we have proven contraction rates for each of the three cases.
\begin{align*}
\begin{cases}
\frac{|\kappa(\rho)-\rho^*|}{1-\kappa(\rho)} \le \frac{|\rho-\rho^*|}{1-\rho}\frac{1-\rho^*}{2-\kappa'(\rho^*)} & \text{if } \rho \ge \rho^* \\
|\kappa(\rho)-\rho^*| \le \kappa'(\rho^*)|\rho-\rho^*| & \text{if } 0 \le \rho \le \rho^*\\
|\kappa(\rho)-\rho^*| \le (1 - k(0))|\rho-\rho^*| & \text{if } \rho \le 0 
\end{cases}
\end{align*}


Let us now define the joint decay rate:
\begin{align*}
\alpha = \max\left\{1 - k(0), \kappa'(\rho^*), \frac{1-\rho^*}{1-\kappa'(\rho^*)}\right\}
\end{align*}
In other words, this is the worst-case rate for any of the above cases. 

Now, let us assume we are starting from initial $\rho_0$ and define $\rho_\ell = \kappa(\rho_{\ell-1})$. One important observation is that if we have $\rho_0 \ge \rho^*$ then by monotonicity of $\kappa$ in the $[0,1]$ range, it will remain the same range, and similarly if $\rho_0\in[0,\rho^*]$ it will remain in the same range. Thus, from that index onwards, we can apply the contraction rate of the respective case. The only case that there might be a transition is if $\rho_0 < 0$. 

Assuming that $\rho_0 < 0$, let $\rho_\ell$ be the first index that we have $\rho_\ell \ge 0$. Thus, from $\rho_0$ to $\rho_\ell$ we can apply the contraction rate of the third case:
\begin{align*}
|\rho_\ell - \rho^*| \le |\rho_0 - \rho^*| \alpha^\ell
\end{align*}

Now, we have two possibilities, either $\rho_\ell \ge \rho^*$ or $\rho_\ell \le \rho^*$. If $\rho_\ell \ge \rho^*$, we can apply the contraction rate of the first case, and if $\rho_\ell \le \rho^*$ we can apply the contraction rate of the second case:
\begin{align*}
    \begin{cases}
|\rho_L - \rho^*| \le |\rho_\ell - \rho^*| \alpha^{L-\ell} & 0\le \rho_\ell\le \rho^*\\
\frac{|\rho_L-\rho^*|}{1-\rho_L} \le \frac{|\rho_\ell-\rho^*|}{1-\rho_\ell}\alpha^{L-\ell} & \rho_\ell\ge \rho^* \\
    \end{cases}
\end{align*}
If we plug in our contraction up to step $\ell$ and use the fact that the norm of the sequence is non-increasing $|\rho_0|\ge \rho_\ell,$ we have
\begin{align*}
    \begin{cases}
|\rho_L - \rho^*| \le |\rho_0 - \rho^*| \alpha^{L} & 0\le \rho_\ell\le \rho^*\\
\frac{|\rho_L-\rho^*|}{1-\rho_L} \le \frac{|\rho_0-\rho^*|}{1-|\rho_0|}\alpha^{L} & \rho_\ell\ge \rho^* \\
    \end{cases}
\end{align*}
We can now take the worst case of these two and conclude that
\begin{align*}
|\rho_L - \rho^*| \le \frac{|\rho_0 - \rho^*|}{1-|\rho_0|} \alpha^L
\end{align*}

So far in the proof, we assumed the existence of $\rho^*$ that obeys $\kappa'(\rho^*)<1.$ We can now prove that such a fixed point exists. It is unique, and it is necessarily in the range $\rho^*\in [0,1].$

\textbf{Positivity, uniqueness, and existence of a globally attractive fixed point}
Here, the goal is to prove there is exactly one point $\rho^*\in[0,1]$ such that $\kappa(\rho^*) = \rho^*$ and $\kappa'(\rho^*) < 1.$ We will prove the properties of positivity, uniqueness, and existence separately.

\textit{Positivity:} Let us assume that $\rho^* \le 0.$ is a fixed point. Then, we can apply the contraction rate proven for Case 3, which shows that $\kappa(\rho^*)\ge \rho^* + k(0) > \rho^*,$ which is a contradiction.

\textit{Uniqueness:} Assume that there are two fixed points $\rho_1$ and $\rho_2$ that satisfy $\kappa'(\rho_1),\kappa'(\rho_2)<1.$ Let us assume wlog that $\rho_1 < \rho_2.$ Then we can invoke the contraction rate proven so far to argue that all points in $(-1,1),$ including  $\rho\in (\rho_1,\rho_2)$ are attracted towards both $\rho_1$ and $\rho_2,$ which is a contradiction. Thus, there can be at most one fixed point.

\textit{Existence of $\rho^*$:} Because $\kappa(1)=1$ the set of all fixed points is non-empty. Let us assume that $\rho^*$ is the first (smallest) fixed point of $\kappa(\rho^*) = \rho^*,$ which because of the positivity result is necessarily $\rho^*>0.$ If we assume that $\kappa'(\rho^*) > 1,$ then in the small $\epsilon$-neighborhood of it $\rho_1 \in(\rho^*-\epsilon,\rho^*)$ we have $\kappa(\rho_1) < \rho_1.$ Because $\kappa(\rho)$ is continuous, and is above identity line at $\rho=0$ and under identity line $\rho=\rho_1,$ there must be a point $0 < \rho_2 < \rho_1$ where it is at identity $\kappa(\rho_2) = \rho_2,$ which is a contradiction with assumption that $\rho^*$ is the smallest fixed point. Thus, we must have $\kappa'(\rho^*) \le 1.$ If we assume that $\kappa'(\rho^*) = 1,$ then the $\kappa$ must align with the identity line from $\rho^*$ to $1,$ which implies that all higher order terms $c_k,k\ge 2$ must be zero, which in turn implies that $\kappa$ is a linear function. This is a contradiction with the assumption that the activation is nonlinear. Thus, we must have $\kappa'(\rho^*) < 1,$ which proves the desired existence. 
\end{proof}

\begin{proof}[Proof of Proposition~\ref{prop:residual_kernel_map}]
First, we ought to prove that in the mean-field regime, the kernel map is transformed according to the equation that is given. To do so, we can consider one layer update. Assume that $X,Y\sim N(0,1)$ with covariance $\E XY = \rho.$ Now, $(X',Y')$ is an independent copy of $(X,Y),$ with the same variance and covariance structure. This is due to the presence of skip connection weights $P.$ Now, defined the joint layer update 
\begin{align*}
\psi(X) &= \sqrt{1-r^2} \phi(X)  + r X',\\
\psi(Y) &= \sqrt{1-r^2} \phi(Y')  + r Y',\\
\implies \kappa_\psi(\rho)&:= \E \psi(X)\psi(Y) \\
&=  (1-r^2) \E \phi(X)\phi(Y) + r^2 \E X' Y' \\
    &= (1-r^2) \kappa(\rho) + r^2 \rho,
\end{align*}
where in the third line we use the independence assumption and the fact that $X'$ and $Y'$ have zero mean. 

Observe that $\kappa_\psi(\rho) = (1-r^2) \kappa(\rho),$ or $\kappa_\psi'(\rho) = (1-r^2) \kappa'(\rho) + r^2.$ Now, let us consider the conditions for the four cases of Theorem~\ref{thm:global_attract}. The decision boundaries are if $\kappa(0)$ is positive or not, if $\kappa'(1)$ is above, equal to, or below $1.$ Now, note that for $r \in (0,1),$ strict positivity of $\kappa_\psi(0) = (1-r^2) \kappa(0)$ remains the same as $\kappa(0).$ Furthermore, $\kappa_\psi'(1)$ is a weighted average of $\kappa'(1)$ and $1.$ Thus, for all values $r\in(0,1),$ it holds $\kappa_\psi'(1)$ is above, equal to, or below $1,$ if and only if $\kappa'(1)$ has the corresponding property. Now, we can turn our focus on the convergence rates. For this, we can focus on $\alpha$ in each one of the four cases. 
\begin{itemize}
    \item If $\kappa(0)=0$ then $\kappa_\psi(0) = r^2 \kappa(0)$ and we have $\alpha = 1 / (2-\kappa_\psi'(0)).$ Now, note that $\kappa_\psi'(0) = (1-r^2)\kappa'(0) +  r^2$ is a weighted average between $\kappa'(0) < 1$ and $1,$ and thus, the larger the value of $r,$ the larger $\alpha$ would be, and the slower the convergence. 
    \item If $\kappa(0)>0$ and $\kappa'(1)<1$ then we have $\kappa_\psi(0)=0$ and $\kappa_\psi'(1) < 1,$ and $\alpha = \kappa_\psi'(1) = (1-r^2) \kappa'(1) + r^2.$ Thus, the larger the residual, the closer $\alpha$ becomes to $1,$ and the slower the convergence will be.  
    \item If $\kappa(0)>0$ and $\kappa'(1)=1,$ then we have the same for $\kappa_\psi,$ and 
    \begin{align*}
        \alpha &= 1 - \kappa_\psi(0) - \kappa_\psi'(0)\\
        &= 1 - (1-r^2) \kappa(0) - (1-r^2)\kappa'(0) - r^2\\
        &= (1-r^2 )( 1 - \kappa(0) - \kappa'(0))
    \end{align*}
    Now, recall that we have $\kappa(0) = c_0^2$ and $\kappa'(0)=c_1^2,$ and thus we have $1-\kappa(0) - \kappa'(0) = \sum_{k=2}^\infty,$ which is necessarily positive for a nonlinear activation:
    \begin{align*}
        \alpha = (1-r^2) \sum_{k=2}^\infty c_k^2.
    \end{align*}
    We can now see that for larger $r,$ $\alpha $ will be smaller, which in this case implies a slower convergence. 
    \item If $\kappa(0)>0$ and $\kappa'(1) > 1$ then same holds for $\kappa_\psi,$ and the convergence rate $\alpha$ .
\end{itemize}

\end{proof}

\begin{proof}[Proof of Proposition~\ref{prop:normalization_kernel_map}]
Recall the assumption that $\phi$ obeys $\E \phi(X)^2 = 1,\, X\sim N(0,1). $ Let us assume that $Z\sim N(0,I_d)$ represents the Gaussian pre-activations from the previous layer We can consider the joint normalization and activation layer $\psi$ in the mean-field regime: 
\begin{itemize}
    \item $\psi = \phi\circ RN$ and $\psi = \phi\circ LN$: Because each element of $Z$ has zero mean and unit variance, in the mean field regime, the sample mean and variances will be equal to the their population counterparts, implying that $LN(Z) = Z$ and $RN(Z) = Z.$ In other words, they act as identity. Thus, in both cases, in the mean-field regime we have $\psi = \phi.$ Thus, the kernel map also remains the same. 
    \item $\psi = RN \circ \phi$: In this case, because of the assumption on activation for Gaussian preactivations we have $\E \phi(Z_i)^2 = 1,$ for all $i=1,\dots, d.$ Because elements of $Z$ are i.i.d., by law of large numbers $\frac1d\|\phi(Z)\|^2$ will converge to its expected value $1.$ Thus, again, the normalization step in RN becomes ineffective, implying that in the mean-field we have $\psi = \phi.$ Thus, the kernel map also remains the same. 
    \item $\psi = LN \circ \phi$: By definition of kernel map of activation, we have  
    \begin{align*}
        \E \phi(Z_i) &= \sqrt{\kappa_\phi(0)}, 
       && \mathrm{Var}(Z_i) = \E \phi(Z_i)^2 - (\E \phi(Z_i))^2 =  1- \kappa_\phi(0)
    \end{align*}
    Thus, again by law of large numbers we will have 
    $$\psi = (\phi - \sqrt{\kappa_\phi(0)})/\sqrt{1 - \kappa_\phi(0)}.$$ Now, if we look at the kernel map of this affine transformation, we have 
    \begin{align*}
    \kappa_\psi(\rho) &= \E \psi(X)\psi(Y)\\
    &= \E \frac{(\phi(X)-\sqrt{\kappa_\phi(0)} )}{\sqrt{\kappa(1)-\kappa(0)}}\frac{(\phi(Y)-\sqrt{\kappa_\phi(0)} )}{\sqrt{1-\kappa_\phi(0)}}\\
    &= \frac{\E \phi(X)\phi(Y) - \kappa_\phi(0)}{1-\kappa_\phi(0)}\\
    &= \frac{\kappa_\phi(\rho)-\kappa_\phi(0)}{1-\kappa_\phi(0)}
    \end{align*}
    which concludes the proof. 
\end{itemize}
\end{proof}

\begin{remark}\label{rem:residual_ODE}
    Another interesting implication is that if we take a look at discrete and continuous kernel dynamics for a very strong residual (when $r\to 0$), the kernel ODE becomes exact. Let us denote the Hermite coefficients of the residual kernel by $\tilde{c}_k.$ For the first term, we have $\tilde{c}_1^2 = r^2 c_1^2 + (1-r^2) 1$ and for all other terms $\tilde{c}_k^2 = r^2 c_k$. Note that as $r \to 0,$ we have $\tilde{c}_1^2\to 1,$ and for all other terms $\tilde{c}_k^2\to 0.$ Thus, in the discrete kernel dynamics, the right-hand side will converge to zero. This implies that in the limit of very strong residuals, the kernel ODE gives the exact solution to the kernel dynamics. 
\end{remark}

\begin{remark}\label{rem:no_negative_geometric}
    Suppose there is an activation $\phi$ with a negative fixed point $\rho^*<0.$ Let $n$ be an integer such that $\rho^* < -1/(n-1).$ Let $x_1,\dots, x_n\in \R^d$ be vectors that have non-zero inner products. Construct an MLP with activation $\phi$ and let $y_1,\dots, y_n$ be the output of this MLP. Thus, if the depth is sufficiently large, the pairwise similarity between each pair will converge to $\rho^*.$ Thus, their output Gram matrix $G = [\langle y_i,y_j\rangle]_{i,j\le n}$  has unit diagonals and off diagonals equal to $\rho^*.$ By analyzing the eigenvalues of $G$, we find: The top eigenvalue corresponding to the all-ones eigenvector is $\lambda_1 = 1 + (n-1)\rho^*$. Because we assumed $\rho^* < -1/(n-1),$, we have $1 + (n-1)\rho^* < 0$, which is a contradiction. Because $G$ by construction must be positive and semi-definite.  
\end{remark}

\begin{corollary}\label{cor:double_exp}
    If Hermite coefficients of $\phi$ has Hermite expansion $\phi = \sum_{k=m}^\infty c_k \he_k,$ then kernel sequence of MLP with activation $\phi$ converges to zero with double exponential rate 
    \begin{align*}
        |\rho_\ell| \le |\rho_0|^{m^\ell}.
    \end{align*}
\end{corollary}

The proof of corollary is by simply observing that, the slowest convergence happens when the weight of coefficients is concentrated at $k=m$ term, and for that case, we can observe that $\|\kappa(\rho)\|\le \|\rho\|^m,$ and by induction over $\ell$ we can prove the claim. 

\section{Validation of the global convergence theorem}\label{sec:experiments}
Here we will provide some numerical validation of the global convergence theorem. 

\begin{remark}\label{rem:act-normalization}
    First off, note that Theorem~\ref{thm:global_attract} requires that the activation functions preserve (do not increase or decrease) the energy of the pre-activations, $\E[\phi(X)^2] = 1,$ for $X\sim N(0,1).$ While some activation functions, namely \texttt{SeLU}~\cite{klambauer2017self}, have this property, we can achieve it for all activation functions by inversely scaling them by a constant factor, equal to $C = \sqrt{\E[\phi(X)^2]}.$ After this step, we can quantify various values relevant to Theorem~\ref{thm:global_attract}. This step is applied for both the results in the table, as well as Figures~\ref{fig:validation_plots} and \ref{fig:validation_plots2}. The scaling constant $C$ for each activation function is shown in Table~\ref{tab:activation_stats}.
\end{remark}

Recall the conditions for each convergence from Theorem~\ref{thm:global_attract}:
\begin{itemize}
    \item $\kappa(0)=0$: Case 1, exponential convergence towards $\rho^*=0.$
    \item $\kappa(0)>0$:
    \begin{itemize}
        \item $\kappa'(1)<1$: Case 2, exponential convergence towards $\rho^*=1.$
        \item $\kappa'(1)=1$: Case 3, polynomial convergence towards $\rho^*=1.$
        \item $\kappa'(1)>1$: Case 4: exponential convergence towards some $\rho^*\in (0,1).$
    \end{itemize}
\end{itemize}
In Table~\ref{tab:activation_stats}, we have quantified the fixed points, relevant quantities, and convergence rates for each activation function. This table shows that for the most popular activation functions that are used in practice, we can quantify a fixed point and explicit rate that their kernel sequence will converge to. 

It is worth noting that because Theorem~\ref{thm:global_attract} provides an upper bound, we cannot directly compare the quantified value of $\alpha$ for different cases, particularly when they correspond to different cases. In other words, because this is a worst-case bound, it is possible that an activation function manifests a much higher convergence than predicted. We will further explore the gaps between the empirical values and the upper bound in the subsequent figures.

\begin{table}[ht]
\centering
\renewcommand{\arraystretch}{1.3}
\begin{tabular}{|c|c|c|c|c|c|c|c|c|c| }
\toprule
$\phi$ & $C$ & $\alpha$ & $\rho_\star$ & $\kappa(\rho_\star)$ & $\kappa(0)$ & $\kappa'(0)$ & $\kappa'(1)$ & $\kappa'(\rho^*)$ & Convergence \\ \midrule 
\texttt{tanh} &0.63 &0.93 &0.00 &0.00 &0.00 &0.93 &1.18 &0.93 &case 1 /Exp.\\ 
\texttt{SeLU} &1.00 &0.97 &0.00 &0.00 &0.00 &0.97 &1.06 &0.97 &case 1 /Exp.\\ 
\texttt{ReLU} &0.71 &0.95 &1.00 &1.00 &0.32 &0.50 &0.95 &0.95 &case 2/Exp.\\ 
\texttt{sigmoid} &0.54 &0.15 &1.00 &1.00 &0.85 &0.15 &0.15 &0.15 &case 2 /Exp. \\ 
\texttt{exp} &2.72 &0.74 &1.00 &1.00 &0.37 &0.37 &1.00 &1.00 &case 3 /Poly.\\ 
\texttt{GELU} &0.65 &0.93 &0.76 &0.76 &0.19 &0.59 &1.07 &0.93 &case 4 /Exp. \\ 
\texttt{CELU} &0.80 &0.97 &0.60 &0.60 &0.04 &0.90 &1.04 &0.97 &case 4 /Exp.\\ 
\texttt{ELU} &0.80 &0.97 &0.60 &0.60 &0.04 &0.90 &1.04 &0.97 &case 4 /Exp.\\ 
\bottomrule
\end{tabular}
\caption{Activation functions: a review of most commonly used activation functions (normalized, see Remark~\ref{rem:act-normalization}), according to Theorem~\ref{thm:global_attract}. In the last column, we have noted the case of convergence according to Theorem~\ref{thm:global_attract}, as well as the speed of convergence (Poly: polynomial, Exp: exponential, not to be confused with  activation $\phi=\texttt{exp}$) in depth. Note that only the \texttt{exp} activation function has polynomial convergence.}
\label{tab:activation_stats}
\end{table}

\subsection{Figures \ref{fig:validation_plots} and \ref{fig:validation_plots2}}
Each row of the figures shows one activation function and each column is dedicated to the following (from left to right). The First column shows the activation function itself, up to some scaling (see Remark~\ref{rem:act-normalization}). In the following, we will describe the last three columns. 

\paragraph{Kernel map and fixed point iterations} The second column shows the kernel map $\kappa$ of each activation function (blue), as defined in Definition~\ref{def:kernel_map}, and the fixed iterations over this kernel (red), as defined by pairs of points 
$$
(\rho_{\ell+1},\rho_\ell), \qquad \qquad \rho_{\ell+1}:=\kappa(\rho_\ell)
$$
where $\rho_0$ indicates the initial value that is arbitrarily chosen (shown as a red dot marker). As the iterations progress, they converge to the fixed point ($\rho^*$, red star marker). The identity map is shown to demonstrate visually how the iterations lead to the fixed point.  

It is worth noting that kernel maps of all activation functions are analytic, i.e., smooth up to an arbitrary degree, even when the activation function itself is non-smooth, which is the case for \texttt{ReLU} and \texttt{SeLU} at $x=0.$ We can also see that despite the non-monotonicity of some activation functions, such as \texttt{GELU}, the kernel map maintains its unique properties, such as having a unique globally attracting fixed point. 

These kernel map plots give important insight into the global convergence of the kernel and the predictions of Theorem~\ref{thm:global_attract}. Intuitively, the fixed point is where the kernel map intersects with the identity line, and the speed at which it converges to the fixed point is inversely related to the slope at $\rho^*.$ For example, \texttt{tanh} and \texttt{SeLU} are close to the identity line, and they exhibit very slow convergence, while \texttt{sigmoid} deviates the most from the identity and shows the fastest convergence rate. Overall, we can see here why deviation from the identity, as captured by $\kappa'(\rho^*),$ and other terms, play an important role in the convergence of the kernel sequence. 

It is worth noting that while \texttt{sigmoid} and \texttt{tanh} are tightly related by the formula $\texttt{tanh}(x) = 2\cdot\texttt{sigmoid}(2x)-1,$ their convergence properties are drastically different, with \texttt{tanh} converging exponentially towards zero (orthogonality or independence bias), while \texttt{sigmoid} converges exponentially towards $1$ (strong similarity bias). We can explain this by observing that the shifting ensures that the activation function has zero-mean postactivations, making $\rho^*$ a fixed point.

\paragraph{Kernel convergence theory vs empirical results.}
The third and fourth columns show the kernel sequence (empirical: blue, theory: red) as a function of depth $\ell.$  The theoretical bound corresponds to Theorem~\ref{thm:global_attract}. In the fourth column, we plot the upper bound provided by the theorem, and in the third column, we use the upper bound on the distance to $\rho^*$ to give a lower or upper bound on the kernel sequence curve.  One of the most important takeaways is that for all activation functions except \texttt{exp}, the distances $|\rho_\ell-\rho^*|,$ shown on the log-scale, decay linearly with depth $\ell.$ This is perfectly aligned with the prediction of Theorem~\ref{thm:global_attract}, because \texttt{exp} is the only activation function that has polynomial convergence (See Table~\ref{tab:activation_stats}). The gap between theory and empirical values corresponds to the worst-case analysis for the global convergence rate.

\begin{figure*}[ht]
    \centering
    \includegraphics[width=\textwidth]{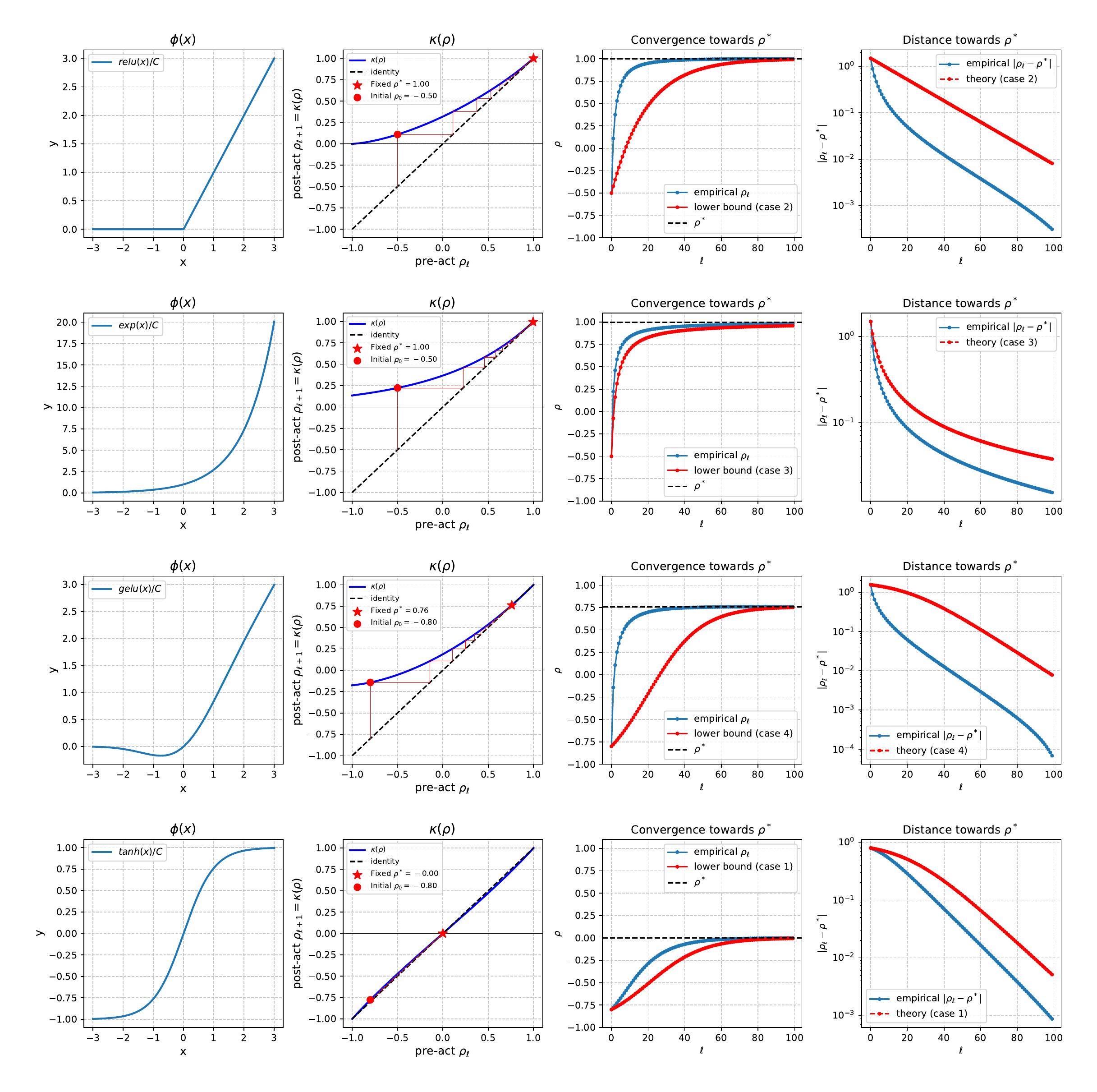}
    \caption{\small Validation of Theorem~\ref{thm:global_attract} Each row corresponds to an activation, scaled down by a factor $C$ to obey $\E \phi(X)^2=1.$. From op to bottom: relu, exp, gelu, tanh. From left, the first column shows the activation, second column shows kernel map, third column shows the kernel sequence vs depth along with the theory prediction, and fourth column shows the distance to the fixed points in for theory and empirical kernels. (Remainder on the next page)}
    \label{fig:validation_plots}
\end{figure*}

\begin{figure*}[ht]
    \centering
    \includegraphics[width=\textwidth]{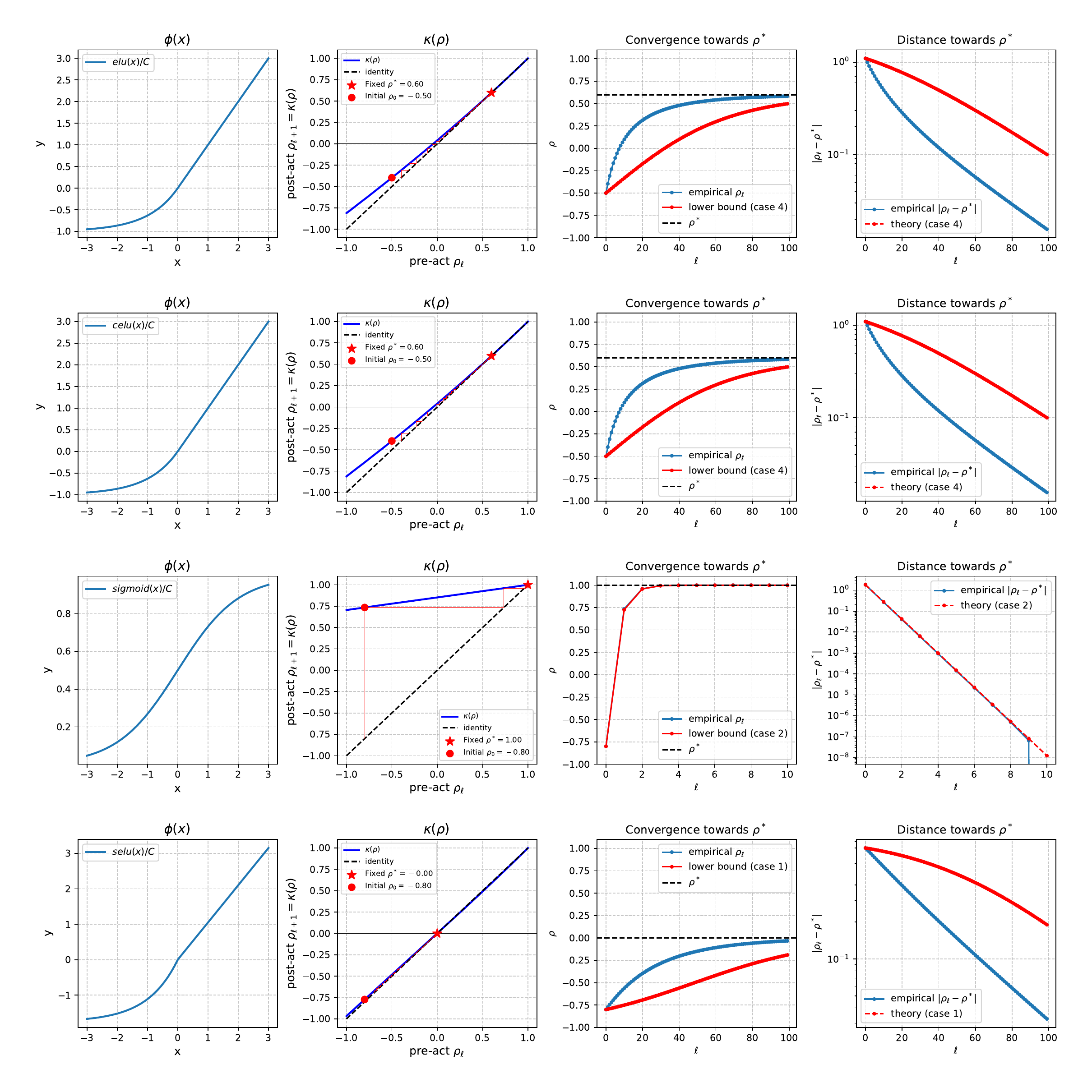}
    \caption{\small Continuation of Figure\ref{fig:validation_plots}, for activations elu, celu, sigmoid, and selu. for the particular case of sigmoid, the errors fall below the numerical precision and cannot be computed.}
    \label{fig:validation_plots2}
\end{figure*}

\begin{figure*}[ht]
    \centering
    \includegraphics[width=\textwidth]{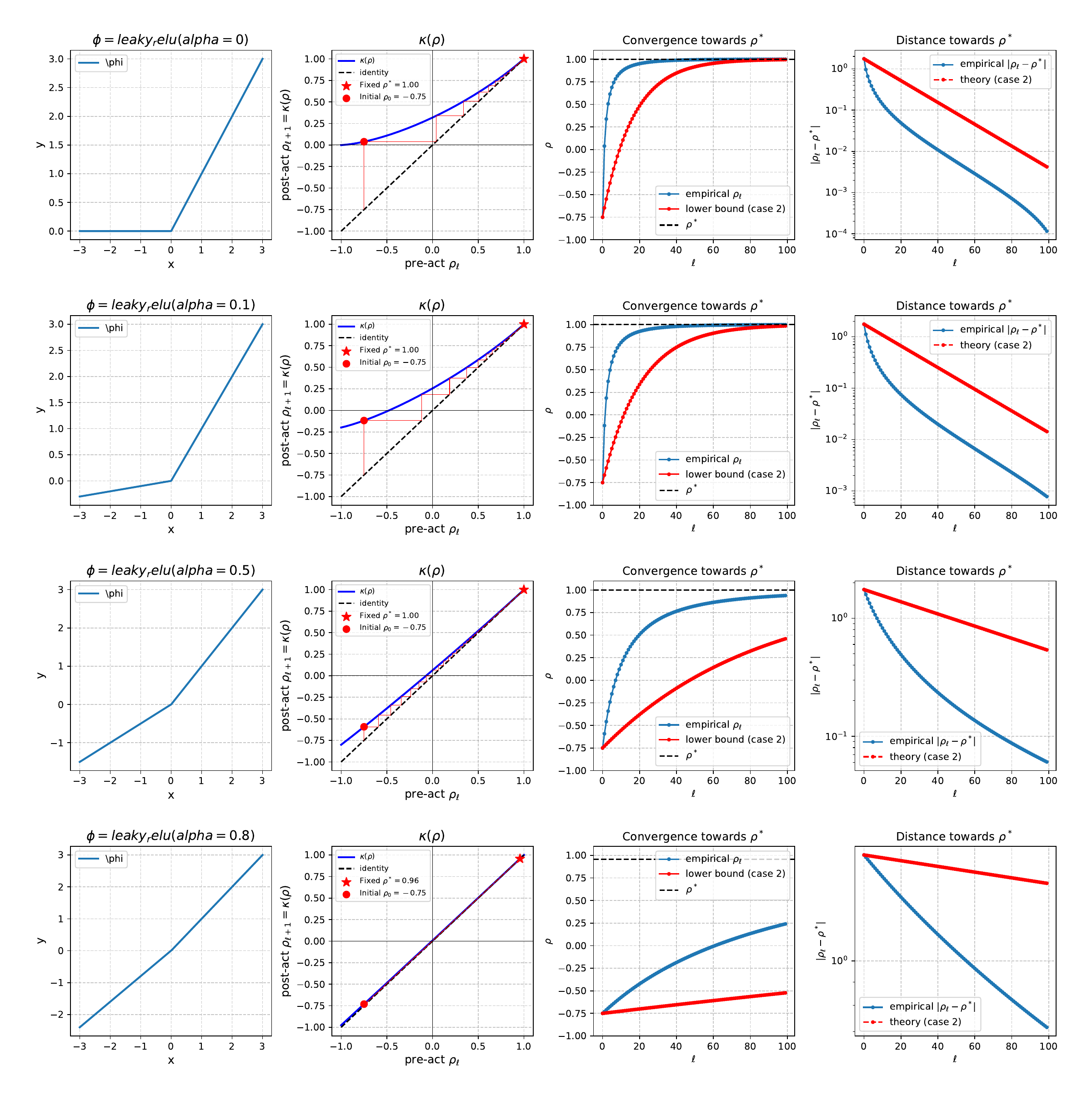}
    \caption{\small Continuation of Figure\ref{fig:validation_plots}, for \texttt{LeakyReLU} with various negative slopes. }
    \label{fig:leaky_relu}
\end{figure*}


\end{document}